\documentclass[letterpaper]{article} % DO NOT CHANGE THIS
\usepackage{aaai2026}  % DO NOT CHANGE THIS
\usepackage{times}  % DO NOT CHANGE THIS
\usepackage{helvet}  % DO NOT CHANGE THIS
\usepackage{courier}  % DO NOT CHANGE THIS
\usepackage[hyphens]{url}  % DO NOT CHANGE THIS
\usepackage{graphicx} % DO NOT CHANGE THIS
\usepackage{natbib}  % DO NOT CHANGE THIS AND DO NOT ADD ANY OPTIONS TO IT
\usepackage{caption} % DO NOT CHANGE THIS AND DO NOT ADD ANY OPTIONS TO IT
\usepackage{algorithm}
\usepackage{algorithmic}
\usepackage{adjustbox}
\usepackage{algorithmic}
\usepackage{amsmath}
\usepackage{array}
\usepackage{amsfonts}
\usepackage{amsthm}
\usepackage{booktabs}
\usepackage{multirow}
\usepackage{subcaption}
\usepackage[version=4]{mhchem}
\newtheorem{theorem}{Theorem}

\usepackage{newfloat}
\usepackage{listings}
\DeclareCaptionStyle{ruled}{labelfont=normalfont,labelsep=colon,strut=off} % DO NOT CHANGE THIS
\floatstyle{ruled}
\newfloat{listing}{tb}{lst}{}
\floatname{listing}{Listing}
\title{LGAN: An Efficient High-Order Graph Neural Network \\via the Line Graph Aggregation}
\author{
    Lin Du\textsuperscript{\rm 1},
    Lu Bai\textsuperscript{\rm 1}\thanks{Corresponding author: Lu Bai},
    Jincheng Li\textsuperscript{\rm 1},
    Lixin Cui\textsuperscript{\rm 2},
    Hangyuan Du\textsuperscript{\rm 3},
    Lichi Zhang\textsuperscript{\rm 4},
    Yuting Chen\textsuperscript{\rm 5},
    Zhao Li\textsuperscript{\rm 6}
}
\affiliations{
    \textsuperscript{\rm 1}School of Artificial Intelligence, Beijing Normal University, Beijing, China\\
    \textsuperscript{\rm 2}School of Information, Central University of Finance and Economics, Beijing, China\\
    \textsuperscript{\rm 3}School of Computer and Information Technology, Shanxi University, Taiyuan, China\\
    \textsuperscript{\rm 4}School of Biomedical Engineering, Shanghai Jiaotong University, Shanghai, China\\
    \textsuperscript{\rm 5}Centre for Learning Sciences and Technologies, The Chinese University of Hong Kong, Hong Kong, China\\
    \textsuperscript{\rm 6}Zhejiang Lab, Zhejiang, China\\
    dulin@mail.bnu.edu.cn, bailu@bnu.edu.cn

}

\usepackage{bibentry}
\begin{document}

\maketitle

\begin{abstract}
Graph Neural Networks (GNNs) have emerged as a dominant paradigm for graph classification. Specifically, most existing GNNs mainly rely on the message passing strategy between neighbor nodes, where the expressivity is limited by the 1-dimensional Weisfeiler-Lehman (1-WL) test. Although a number of $k$-WL-based GNNs have been proposed to overcome this limitation, their computational cost increases rapidly with $k$, significantly restricting the practical applicability. Moreover, since the $k$-WL models mainly operate on node tuples, these $k$-WL-based GNNs cannot retain  fine-grained node- or edge-level semantics required by attribution methods (e.g., Integrated Gradients), leading to the less interpretable problem. To overcome the above shortcomings, in this paper, we propose a novel Line Graph Aggregation Network (LGAN), that constructs a line graph from the induced subgraph centered at each node to perform the higher-order aggregation.  We theoretically prove that the LGAN not only possesses the greater expressive power than the 2-WL under injective aggregation assumptions, but also has lower time complexity. Empirical evaluations on benchmarks demonstrate that the LGAN outperforms state-of-the-art $k$-WL-based GNNs, while offering better interpretability.
\end{abstract}
% Uncomment the following to link to your code, datasets, an extended version or similar.
% You must keep this block between (not within) the abstract and the main body of the paper.
% \begin{links}
%     \link{Code}{https://aaai.org/example/code}
%     \link{Datasets}{https://aaai.org/example/datasets}
%     \link{Extended version}{https://aaai.org/example/extended-version}
% \end{links}
\section{Introduction}

Recently, Graph Neural Networks (GNNs) have proven to be powerful tools for graph-structured data analysis across various domains~\cite{bai2023,cui2024,qin2025}. Most GNNs follow the message-passing framework, where each node updates its representation by aggregating information from its neighbors, thereby incorporating both node features and graph topology.

A fundamental limitation of such Message Passing Neural Networks (MPNNs) lies in their inability to distinguish certain non-isomorphic graphs, which has been proven to be at most as powerful as the 1-dimensional Weisfeiler-Lehman (1-WL) test~\cite{xu2018powerful}. This limitation is usually caused by ignoring interactions among neighbors, so that the MPNNs cannot distinguish the subtle structural differences like triangles.
\begin{figure}[t]
\centering
\includegraphics[width=1\columnwidth]{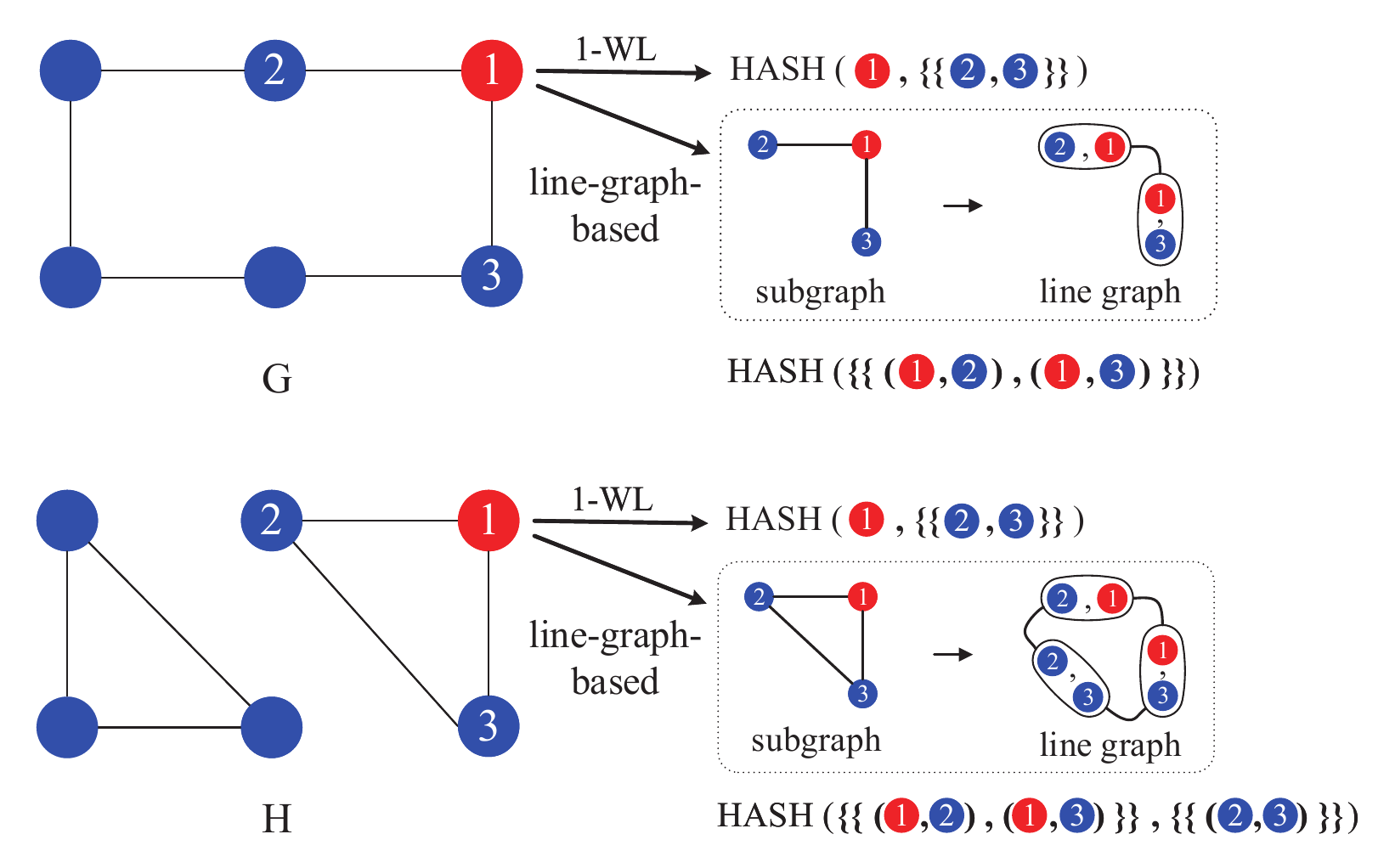}
\caption{Motivating example: why 1-WL fails and line-graph-based aggregation succeeds.}
\label{fig1}
\end{figure}
For instance, Figure~\ref{fig1} exhibits a pair of graphs $G$ and $H$, that are non-isomorphic but indistinguishable under the $1$-WL test. The red target node $1$ in both graphs aggregates messages from an identical multiset of neighbors (i.e., nodes $2$ and $3$), leading to identical hash updates despite distinct substructures.

To overcome the expressivity limitations of MPNNs, researchers have developed a number of alternative GNNs based on the higher-order $k$-WL test~\cite{cai1992optimal,shervashidze2011weisfeiler} or stronger variants such as the $k$-FWL (Folklore Weisfeiler-Lehman) test. Since the $k$-WL test provides a more powerful framework for capturing the structural similarity of graphs, these $k$-WL-based GNNs simulating the higher-order isomorphism tests can  achieve better performance on challenging graph learning tasks~\cite{morris2019goneural, maron2019provably,morris2020weisfeiler,bodnar2021weisfeiler}. However, they rely on operations over node tuples or sets. While increasing $k$ enables capturing higher-order node interactions~\cite{sato2020survey,huang2021short},  this design also results in greater computational overhead and reduced interpretability.

On the other hand, the line graphs from the classical graph theory have recently offered another promising direction for improving the expressive power of GNNs. For a line graph, each node corresponds to an edge of the original graph structure, and each edge is formed if its connected nodes corresponding to the original edges sharing the same original node. This construction inherently captures the higher-order structural information that is not directly accessible in the original graph. As shown in Figure~\ref{fig1}, the line graph of the induced subgraph centered at node $1$ encodes not only the node-neighbor relations $(1,2)$ and $(1,3)$, but also the neighbor-neighbor interactions $(2,3)$, enabling more expressive and structure-aware aggregation. Thus, this enriched representation can naturally distinguish graphs $G$ and $H$ with just one iteration. Motivated by this observation, the line-graph-based neural networks have attracted growing attention, particularly in edge-centric tasks such as link prediction~\cite{zhang2023line,liang2025line} and community detection~\cite{chen2017supervised}. However, their theoretical connection to the $k$-WL test is still not clear, influencing the employments for other graph-level tasks.

To address the above shortcomings, we propose a novel GNN framework, namely the \textbf{Line Graph Aggregation Network (LGAN)}, for graph classification. Our key idea is to perform the information aggregation associated with the line graphs, that are constructed from the induced subgraph centered at each node. The proposed method can employ either the structural relationship between an original graph and its line graph or the favorable properties of the line graph. Overall, \textbf{the main contributions} are threefold.
\begin{itemize}
    \item[1.] \textbf{We introduce the LGAN, a novel line-graph-based GNN for graph-level tasks.} Unlike the previous line-graph-based methods focusing on edge-level tasks, the LGAN constructs line graphs over the node-centered subgraphs and further aggregates their messages to represent the whole graph. Moreover, the LGAN employs a dual aggregation design over the line graph to reinforce its isomorphic correspondence with the original graph.
    \item[2.] \textbf{We prove the expressive power of the LGAN and its relationship to the $k$-WL test.} We prove that the LGAN surpasses the 2-WL test, under the assumption that aggregation-related functions are injective. Moreover, we show that the LGAN simulates the behavior of 2-FWL in a localized manner, while reducing the time complexity from cubic to nearly linear on sparse graphs.
    \item[3.] \textbf{We empirically validate the performance of the LGAN on graph classification tasks.} The LGAN consistently outperforms or matches state-of-the-art baselines including the $k$-WL-based models. Moreover, the LGAN enables the visualization of important substructures in both synthetic and real-world graphs through edge-level attributions.
\end{itemize}

\section{Preliminary Concepts}

\subsection{The Basic Notations}

Let  $G(V, E, X_V)$ be an undirected graph with node set $V$, edge set $E$, and node feature matrix $X_V = \{\mathbf{x_v} \in \mathbb{R}^d \mid v \in V\}$. The number of nodes and edges in $G$ are denoted by $n = |V|$ and $m = |E|$, respectively. The degree of a node $v \in V$ is denoted by $d_{v}$. The neighborhood of a node $v$ is defined as $\mathcal{N}(v) = \{u \in V \mid (v, u) \in E\}$. A pair of isomorphic graphs $G$ and $H$ are denoted as $G \cong H$. There exists a bijection $\varphi : V(G) \rightarrow V(H)$ such that $(u, v) \in E(G)$  if and only if $(\varphi(u), \varphi(v)) \in E(H)$. $\{\!\{ \cdot \}\!\}$ denotes a multiset, a set that allows repeated elements.

\subsection{The Weisfeiler-Lehman Test}
We commence by introducing the \textbf{1-WL} test~\cite{weisfeiler1968reduction}, which is a well-known algorithm for graph isomorphism test. Given a labeled graph $(G,\ell)$, the algorithm proceeds by iteratively updating the node representations based on the local neighborhood information. Initially, each node is assigned a color by hashing its input features. At iteration \( l > 0 \), each node \( v \in V \) updates its color by aggregating various colors of its multiset neighbors from the previous iteration and applying a hash function to this multiset with its previous color. The update rule is given by
\begin{equation}\label{eq1}
c^{(l)}_{v} = \mathrm{HASH}\left( c^{(l-1)}_{v}, \{\!\{ c^{(l-1)}_{u} \mid u \in \mathcal{N}(v) \}\!\} \right).
\end{equation}
where $c^{(l)}_v$ denotes the color (i.e., the label) of node $v$ at iteration $l$, and $\mathrm{HASH}(\cdot)$ is an injective function to preserve distinctions between different inputs. This process is repeated until the color assignments converge, i.e., no further changes occur across iterations. It can be observed that the $1$-WL algorithm cannot distinguish all non-isomorphic graphs, as illustrated in Figure~\ref{fig1}.

We next introduce the \textbf{k-WL} test, a natural generalization of the $1$-WL algorithm. Unlike the $1$-WL that colors individual nodes, the $k$-WL test operates on the $k$-tuples of nodes and refines their labels by incorporating the structure of the induced $k$-dimensional neighborhoods. Given a labeled graph $(G, \ell)$, the algorithm first assigns initial colors to all $k$-tuples of nodes based on their isomorphism types and labels. At each iteration $l > 0$, the color of a $k$-tuple $\mathbf{v} = (v_1, \dots, v_k) \in V^k$ is updated by aggregating the colors of its neighboring tuples. Specifically, for each position $i \in \{1, \dots, k\}$, we compute the multiset of colors of all tuples obtained by replacing the $i$-th element of $\mathbf{v}$ with every possible node $w \in V$, and then update the color of $\mathbf{v}$. This process is defined as
\begin{subequations}\label{eq2}
    \begin{flalign}
        &c^{(l)}_{\mathbf{v},i} \leftarrow
        \{\!\{ c^{(l-1)}_{\mathbf{v}_{(i \rightarrow w)}} \mid w \in V \}\!\}, \label{eq2sub1}\\
        &c^{(l)}_{\mathbf{v}} = \mathrm{HASH} \left( c^{(l-1)}_{\mathbf{v}}, c^{(l)}_{\mathbf{v},1}, c^{(l)}_{\mathbf{v},2}, \dots, c^{(l)}_{\mathbf{v},k} \right). \label{eq2sub2}
    \end{flalign}
\end{subequations}
where $\mathbf{v}_{(i \rightarrow w)}$ denotes the $k$-tuple resulting from replacing the $i$-th entry of $\mathbf{v}$ with node $w$, $c^{(l-1)}_{\mathbf{v}_{(i \rightarrow w)}}$ is the color assigned to the modified tuple at the previous iteration $(l-1)$, $c^{(l)}_{\mathbf{v},i}$ denotes the multiset of colors obtained by varying the $i$-th element of $\mathbf{v}$, and $c^{(l)}_{\mathbf{v}}$ is the updated color of tuple $\mathbf{v}$ at iteration $l$. The process repeats until the coloring stabilizes.

We then describe the \textbf{k-FWL} test, a variant of the $k$-WL test that is often referred to the Folklore version. Given a labeled graph $(G, \ell)$, each $k$-tuple $\mathbf{v} = (v_1, \dots, v_k) \in V^k$ is initially colored according to its isomorphism type and node labels. At each iteration $l > 0$, the color of a $k$-tuple $\mathbf{v}$ is updated by first constructing  a $k$-dimensional color vector for each node $w \in V$, and then mapping these vectors together with the previous color into a new color. Specifically, 
\begin{subequations}\label{eq3}
    \begin{flalign}
        &c^{(l)}_{\textbf{v},w} \leftarrow
        \left(  c^{(l-1)}_{\textbf{v}_{(1 \rightarrow w)}} , \dots,  c^{(l-1)}_{\textbf{v}_{(k \rightarrow w)}} \right),   \label{eq3sub1}\\
        &c^{(l)}_\textbf{v} = \mathrm{HASH} \left( c^{(l-1)}_\textbf{v},  \{\!\{ c^{(l)}_{\textbf{v},w} \mid  w \in V \}\!\}    \right). \label{eq3sub2}
    \end{flalign}
\end{subequations}
Here, $c^{(l)}_{\mathbf{v},w}$ denotes the $k$-dimensional color vector obtained by collecting the colors $c^{(l-1)}_{\mathbf{v}_{(i \rightarrow w)}}$ of modified tuples $\mathbf{v}_{(i \rightarrow w)}$, and $c^{(l)}_{\mathbf{v}}$ is the updated color of tuple $\mathbf{v}$ at iteration $l$, computed by hashing its previous color with the multiset of these vectors. The iteration continues until convergence.

Note that, for all $k \geq 2$, the expressive power of the $k$-FWL test matches that of the $(k+1)$-WL test~\cite{grohe2015pebble,grohe2017descriptive,grohe2021logic}.

\subsection{The Line Graph}
Given an undirected graph $G(V, E)$, its \textbf{line graph} $L(G)$ is a dual graph whose node set corresponds to the edge set of $G$, i.e., $V(L(G)) = E(G)$. For $L(G)$, a pair of nodes $e_1 = (u_1, v_1)$ and $e_2 = (u_2, v_2)$ are connected by an edge, if and only if the corresponding edges in $G$ share at least one common endpoint. The line graph has some important properties that illustrate its relationship to the original graph. In particular, if the original graph $G$ is connected, then its line graph $L(G)$ is also connected. And if two graphs $G$ and $H$ are isomorphic, i.e., $G \cong H$, then their line graphs are also isomorphic, i.e., $L(G) \cong L(H)$. Moreover, an important classical result from the graph theory~\cite{whitney1992congruent} states the following isomorphism property of the line graph.
\begin{theorem}[Whitney's Isomorphism Theorem]\label{thm:whitney}
Let $G_1$ and $G_2$ be two connected graphs that are neither a triangle ($K_3$) nor a claw graph ($K_{1,3}$). If their line graphs are isomorphic, i.e., $L(G_1) \cong L(G_2)$, then the original graphs are also isomorphic, i.e., $G_1 \cong G_2$.
\end{theorem}
The above theorem serves as an important theoretical foundation for the design of our work. Specifically, this theorem implies that, except for the special case $K_3$ and $K_{1,3}$, the structure of a graph can be uniquely reconstructed from its line graph. Our model will handle these exceptions via a tailored aggregation design, thus ensuring the theorem's necessity and sufficiency within our framework. %  (see details in the later section).

\section{Related Works}

\subsection{The $k$-WL-based GNNs}
In recent years, several higher-order methods have been proposed to overcome the expressivity limitations of MPNNs. Among them, the $k$-WL-based GNNs are most relevant to our work. Notably,~\citet{morris2019goneural} introduced the $k$-dimensional GNNs (k-GNNs), that improve the scalability by simulating the set-based $k$-WL, but lose fine-grained structural information.~\citet{maron2019provably} proposed Provably Powerful Graph Networks (PPGN), which bypass the combinatorial $k$-WL simulation by directly applying global matrix multiplications to encode pairwise interactions, but lack localized message passing to explicitly model substructures such as triangles.~\citet{morris2020weisfeiler} proposed a localized variant of the $k$-WL ($\delta$-$k$-LGNN) to reduce the computational overhead, yet it remains mainly effective on sparse graphs.~\citet{bodnar2021weisfeiler} introduced CW Networks, which transform the original graph into a cell complex---a topological abstraction that is less intuitive than structures like line graphs. 
Overall, these models show that $k$-order invariant or equivariant GNNs can achieve expressivity equivalent to the $k$-WL isomorphism test.

% ~\citet{bodnar2021weisfeiler} adopted a transformation-based approach (CW Networks) for capturing higher-order relations. This method  transforms the original graph into an alternative structure to simulate the $k$-WL expressivity. Specifically, it converts the graph into a cell complex, that is a topological abstraction and is less interpretable than intuitive structures like line graphs. In summary, these models establish that $k$-order invariant or equivariant GNNs can attain expressive power equivalent to the $k$-WL isomorphism test.

Although these $k$-WL-based models capture higher-order structural information, their computational cost grows rapidly with increasing $k$. Even for moderate values such as $k{=}3$, the memory and runtime costs become intractable on real-world graphs. Moreover, models defined on node tuples or sets tend to dilute node-level semantics, limiting interpretability in practice. To address these issues, we will propose a novel LGAN model, which performs line-graph-based local aggregation to emulate the $k$-WL while preserving effective node-level representations.

% Although these $k$-WL-based models can capture effective structural information, their computational efficiency tends to be burdensome and grows rapidly with the increasing order $k$. Even for moderate values such as $k = 3$, the memory and runtime overhead can still become intractable on real-world graphs. Furthermore, some of these $k$-WL-based GNNs are defined based on the tuples or sets of nodes, tending to dilute the node-level semantics and limiting the interpretability of such models in practical applications. To overcome the shortcomings of existing $k$-WL-based GNNs, we will propose a novel LGAN model by performing the line-graph-based local aggregation to emulate the $k$-WL, preserving more effective node-level representations.

\subsection{The Line-Graph-Based Neural Networks}
Line-graph-based neural networks have gained increasing attention for leveraging line graphs to model edge relations~\cite{cai2021line,zhang2023line,liang2025line}. These methods typically reformulate the edge prediction on the original graph as a node prediction task on its line graph. Beyond encoding the edge-centric relationships, the line graphs possess several well-established theoretical properties that make them particularly suitable for simulating higher-order isomorphism tests such as the $k$-WL. For example, the line graphs are tightly connected to the fundamental graph properties like connectivity and isomorphism~\cite{whitney1992congruent}, and they support the linear-time complexity algorithms for reconstructing the original graph~\cite{roussopoulos1973max,lehot1974optimaldetectlinegraph}. In addition, several studies in hypergraph modeling have shown that the line-graph-based representations are advantageous for capturing complex relational structures~\cite{ICPR2014,2016pr}.

However, existing line-graph-based networks mainly focus on edge-level tasks, without fully leveraging the structural properties. In contrast, our LGAN  targets graph-level classification with a tailored aggregation mechanism to capture the structural information. Furthermore, we will bridge the gap between line graphs and the $k$-WL tests by formally analyzing the expressivity of the proposed framework.

\begin{figure*}[t]
\centering
\includegraphics[width=0.85\textwidth]{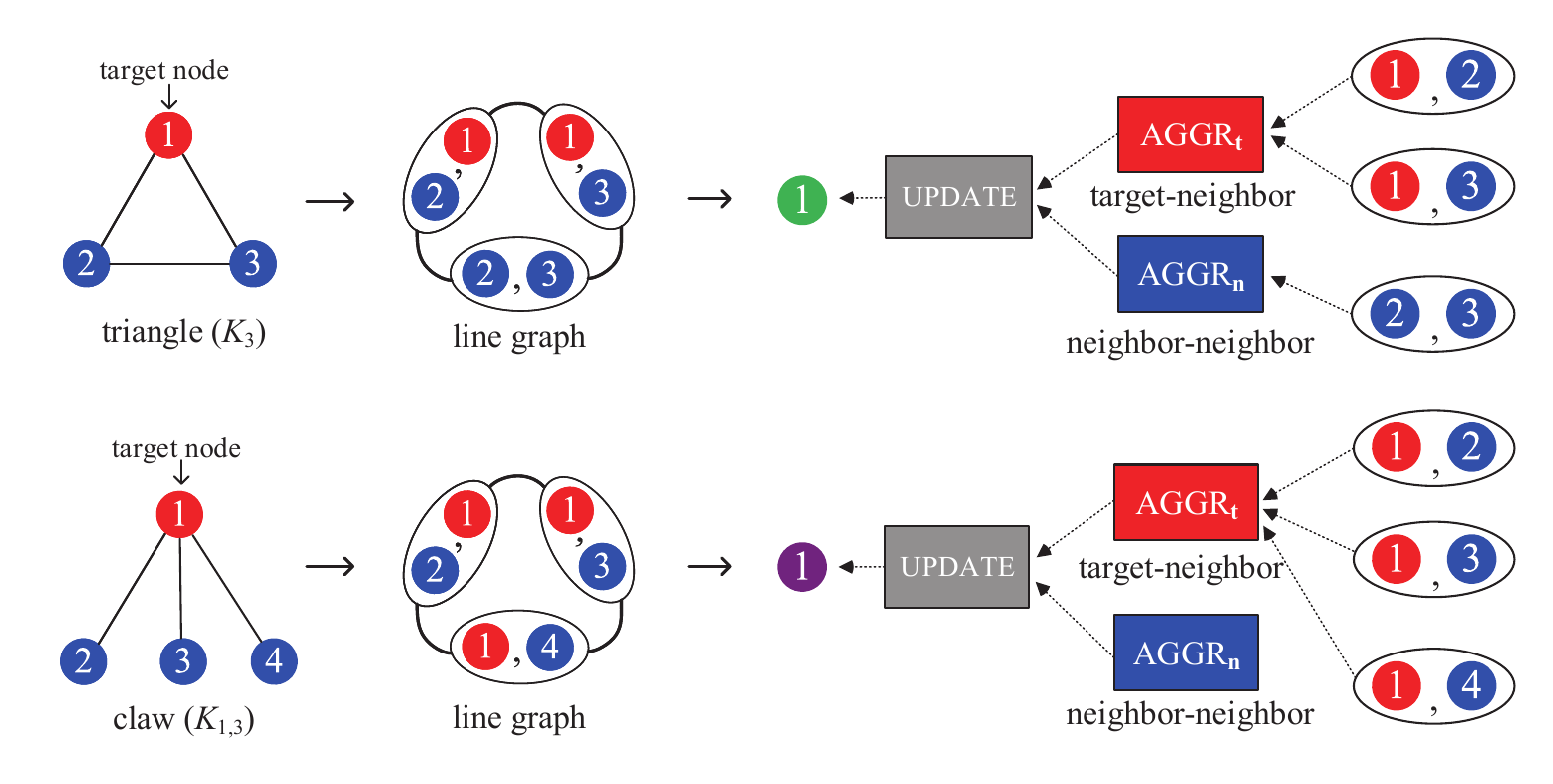} % Reduce the figure size so that it is slightly narrower than the column.
\caption{Overview of the Line Graph Aggregation Network (LGAN). For each target node, the LGAN constructs the induced subgraph, transforms it into a line graph, and performs two relation-specific aggregations: $\mathrm{AGGR}_{t}$ (red) for target--neighbor node pairs (i.e., edges sharing the same target node) and $\mathrm{AGGR}_{n}$ (blue) for neighbor--neighbor node pairs, followed by an \textsc{Update} fusion.
The examples on $K_{3}$ and $K_{1,3}$ illustrate how LGAN resolves the exceptional cases in Whitney's \textbf{Theorem~\ref{thm:whitney}}.}
\label{fig2}
\end{figure*}

\section{The Proposed LGAN Model}\label{sec:model}
In this section, we give the detailed definition of the proposed LGAN model. As illustrated by Figure~\ref{fig2}, for each target node,  the proposed LGAN performs the localized message passing, by first extracting the $1$-hop subgraph around each target node and constructing a line graph over this subgraph. Then the features through both target-neighbor and neighbor-neighbor interactions are aggregated to update the representation of the node. Finally, a multi-layer stacking scheme and global readout function are also employed for graph-level tasks. Below, we define these core components.

\subsection{The Line Graph Construction}

Given an input graph $G(V, E)$, the LGAN updates the hidden representation of each target node $t \in V$ at layer $l$, denoted by $\mathbf{h}^{(l)}_t \in \mathbb{R}^d$. It first constructs an induced subgraph $G_t = (V_t, E_t)$ centered at $t$, where $V_t = \{t\} \cup \mathcal{N}(t)$ and $E_t = \{(u, v) \in E \mid u, v \in V_t\}$. Then, the LGAN builds the line graph $L(G_t) = (E_t, \mathcal{E}_t)$, where each node represents an edge in $G_t$, and two nodes $e_1, e_2 \in E_t$ are connected in $\mathcal{E}_t$ if they share a common endpoint in $G_t$.

\subsection{The Relation-Specific Aggregation}
Let $\mathbf{h}^{(l-1)}_{\{u,v\}}$ denote the representation of a node in the line graph $L(G_t)$ at layer $l{-}1$, corresponding to the node pair $\{u,v\} \in E(G_t)$ in the original graph.  It is computed as
\begin{equation}\label{eq4}
    \mathbf{h}^{(l-1)}_{\{u,v\}} = \mathrm{COMB}_{\text{pair}} \left( \mathbf{h}^{(l-1)}_u , \mathbf{h}^{(l-1)}_v \right),
\end{equation}
where $\mathrm{COMB}_{\text{pair}}$ denotes a symmetric combination function (e.g., element-wise summation) for mapping unordered node pairs to  line graph node representation. 

Then, the LGAN defines two relation-specific aggregations over the line graph $L(G_t)$ as follows.
\begin{itemize}
    \item \textbf{Target--Neighbor Aggregation:} Aggregates features of node pairs $\{t, p\}$ incident to the target node $t$.
    \item \textbf{Neighbor--Neighbor Aggregation:} Aggregates features of node pairs $\{p, q\}$ among  the neighbors of $t$.
\end{itemize}

These aggregations correspond to the red and blue rectangles in Figure~\ref{fig2}, labeled as $\textsc{AGGR}_t$ and $\textsc{AGGR}_n$, respectively. To summarize the \textbf{LGAN}, we formally express the updated representation of the target node $t$ at layer $l$ as
\begin{equation}
    \small
    \mathbf{h}^{(l)}_t = \phi \left(\!
        \mathrm{AGGR}_t\left(\! \{\!\{ \mathbf{h}^{(l-1)}_{\{t,p\}}  \}\!\} \!\right) ,
        \mathrm{AGGR}_n\left(\! \{\!\{ \mathbf{h}^{(l-1)}_{\{p,q\}}  \} \!\} \!\right)
    \!\right),
    \label{eq:target-update}
\end{equation}
where $\phi$ is a learnable update function fusing the two relation-specific aggregated features. Its inputs are the outputs of $\mathrm{AGGR}_t$ and $\mathrm{AGGR}_n$, which are aggregation functions (e.g., sum) operating over multisets of features incident to the target-neighbor pairs and neighbor-neighbor pairs, respectively. Specifically, we adopt concatenation to retain maximal information from both branches. Other fusion strategies (e.g., the element-wise sum, mean, or attention-based gating) can also be applied within the framework.

To further enhance the information preservation and gradient flow, we also introduce a residual variant, termed as \textbf{LGAN-\emph{res}}, which incorporates the previous node representation via a residual connection~\cite{he2016deep}. Specifically, we add a linearly transformed residual from the previous layer to the fused aggregation result.
For the LGAN-\emph{res}, the updated representation of the target node $t$ is given by
{\small
    \begin{flalign}
&\mathbf{z}_t^{(l)} =
    \psi \left(\!
        \mathrm{AGGR}_t\left( \{\!\{ \mathbf{h}^{(l-1)}_{\{t,p\}} \} \!\} \right) ,
        \mathrm{AGGR}_n\left( \{\!\{ \mathbf{h}^{(l-1)}_{\{p,q\}} \} \!\} \right)
    \!\right), \\
&\mathbf{h}_t^{(l)} = \phi'\!\left( W^{(l)} \cdot \mathbf{h}_t^{(l-1)} + \mathbf{z}_t^{(l)} \right),
\label{eq:residual-update}
\end{flalign}
}where $\phi'$ is an additional multi-layer perceptron (MLP) applied after the residual summation.  The residual path includes a linear projection to match dimensions if necessary. $\psi$, $\mathrm{AGGR}_t$, and $\mathrm{AGGR}_n$ retain the same roles as in Eq.~\eqref{eq:target-update}, maintaining permutation invariance and injectivity.

For the cases where the input graph $G$ contains isolated nodes or nodes with no incident edges (i.e., $E(G_t) = \emptyset$), the corresponding line graph $L(G_t)$ becomes empty, resulting in no available edge features for aggregation. In such scenarios, the fused message $\mathbf{z}_t$  can be set to the zero vector. Consequently, the proposed LGAN-\emph{res} naturally reduces to a variant of GIN~\cite{xu2018powerful}, where the target node is updated solely based on a learnable projection of its previous representation, i.e., $\mathbf{h}_t^{(l)} = \phi'(W^{(l)} \cdot \mathbf{h}_t^{(l-1)})$. This fallback behavior preserves stability and ensures that isolated nodes still receive meaningful updates.

\subsection{The Multi-layer LGAN and Readout}
While the above derivation focuses on a single LGAN layer, we stack $L$ such layers to build a deep architecture as
\[
    \mathbf{H}^{(0)} = \mathbf{X}, \quad
    \mathbf{H}^{(l)} = \mathrm{LGANLayer}\bigl( \mathbf{H}^{(l-1)} \bigr), \; l = 1,\dots,L ,
\]
where $\mathbf{H}^{(l)}$ denotes the node representation matrix at  layer $l$. Then we adopt a skip-cat~\cite{xu2018skipcat} strategy, where the representations from all layers are concatenated before the final readout. The process is given by
\[
\mathbf{H}^{\text{skip}} = \left[ \mathbf{H}^{(1)} \,\|\,  \cdots \,\|\, \mathbf{H}^{(L)} \right], \quad
\mathbf{h}_G = \text{READOUT}(\mathbf{H}^{\text{skip}}),
\]
where $\|$ denotes concatenation along the feature dimension, $\mathbf{H}^{\text{skip}}$  is the concatenated node representation matrix from all $L$ layers,
and $\mathbf{h}_G$ is the final graph-level representation obtained by applying a READOUT function  over all nodes.

Note that, the aggregation, update and readout functions of the LGAN follow the DeepSets framework~\cite{zaheer2017deep}, using the sum-based aggregation and MLPs to achieve the permutation invariance and injectivity. This ensures the maximal expressive power~\cite{xu2018powerful}, while maintaining simplicity and efficiency.

\section{Theoretical Properties of the LGAN}
We analyze the expressive power and computational complexity of the proposed LGAN with general cases, where the line graph aggregation is feasible. We prove that it surpasses the  $2$-WL test under standard injectivity assumptions, while implicitly simulating the $2$-FWL behavior through the localized aggregation. We also demonstrate that the LGAN achieves the linear time complexity on sparse graphs.
\subsection{The Expressive Power of the LGAN}
To formalize the expressive advantage of the LGAN, we compare it with the set-based 2-WL test, a widely used and canonical variant operating on unordered node pairs with multiset aggregations. The following theorem shows that the LGAN is at least as powerful as the set-based 2-WL and can even distinguish some graphs that 2-WL cannot.

\begin{theorem}[The LGAN is more expressive than the set-based 2-WL]\label{LGAN surpasses 2-WL}
 Let $A: \mathcal{G} \rightarrow \mathbb{R}^d$ be an $L$-layer LGAN, the following components are injective at every layer, i.e., 1) the node-pair encoder $\mathrm{COMB}_{\text{pair}}$ mapping unordered node pairs to the representation of the corresponding line graph node, 2) the multiset aggregators $\mathrm{AGGR}_t$ and $\mathrm{AGGR}_n$ (e.g., sum-based DeepSets), 3) the fusion and update functions $\psi$, $\phi$ and $\phi'$, and 4) the graph-level readout function applied to the multiset of node representations. For any graphs $G$ and $H$ such that the induced subgraph $G_t$ (resp. $H_t$) for every node $t$ contains at least one edge (i.e., the corresponding line graphs are non-empty), the following holds.
\begin{itemize}
    \item If the set-based $2$-WL test distinguishes $G$ and $H$, then the LGAN also maps them to different graph representations, i.e., $A(G) \ne A(H)$.
    \item There exist graphs $G$ and $H$ such that the set-based $2$-WL test fails to distinguish them, but the LGAN produces different representations, i.e., $A(G) \ne A(H)$.
\end{itemize}
\end{theorem}

\begin{proof}
We first prove by induction that any distinction made by the set-based 2-WL is reflected in  the node representations of the LGAN, so that the final injective readout function further leads to different graph-level representations. We then demonstrate that the LGAN succeeds in distinguishing certain graphs where the 2-WL test fails through a counterexample. To commence, let $c_{\{a,b\}}^{(l)}$ denote the color assigned by the 2-WL algorithm to the unordered node pair $\{a, b\}$ at layer $l$, and let $\mathbf{h}_t^{(l)}$ denote the hidden representation of a target node $t$ produced by the LGAN at the same layer. We assume the inductive hypothesis that for some layer $l \geq 0$, if there exists a node pair $\{a, b\}$ such that $c_{\{a,b\}}^{(l)}(G) \neq c_{\{a,b\}}^{(l)}(H)$, then there exists a node $t \in \{a, b\}$ for which $\mathbf{h}_t^{(l+1)}(G) \neq \mathbf{h}_t^{(l+1)}(H)$.

We first consider the base case $l = 0$. Suppose the initial 2-WL color  $c_{\{a,b\}}^{(0)}$ encodes both the initial features of nodes $a$ and $b$ (i.e., $\mathbf{x}_a$ and $\mathbf{x}_b$), and their connectivity. If $c_{\{a,b\}}^{(0)}(G) \neq c_{\{a,b\}}^{(0)}(H)$, two scenarios are detailed below.

\emph{\textbf{Scenario 1}} represents the structural difference, such as when $a$ and $b$ are connected in $G$ but not in $H$.
Thus, the node-pair representation $\mathbf{h}_{\{a,b\}}^{(0)}$ of $G$ exists and is included in the multiset as given by
\begin{equation}
 \mathbf{h}_{\{a,b\}}^{(0)} \in \left\{\!\!\left\{ \mathbf{h}_{\{a,w\}}^{(0)} \mid w \in \mathcal{N}_G(a) \right\}\!\!\right\} = {M}_{a}{(G)},
\end{equation}
where ${M}_{a}{(G)}$ is the multiset input to  $\mathrm{AGGR}_t$ for the target node $a$ (cf. Eq.~\ref{eq:target-update}).
Since $\mathbf{h}_{\{a,b\}}^{(0)} \notin {M}_{a}{(H)}$ and $\mathrm{AGGR}_t$ is injective, it follows that
\begin{equation}
\mathrm{AGGR}_t\left({M}_{a}{(G)}\right) \neq \mathrm{AGGR}_t\left({M}_{a}{(H)}\right).
\end{equation}
This difference propagates to the node representations updated by the LGAN at the next layer (layer 1), given by
\begin{equation}
\phi( \mathrm{AGGR}_t(M_a{(G)}), \cdot ) \neq \phi( \mathrm{AGGR}_t(M_a{(H)}), \cdot ),
\end{equation}
where \(\phi\) denotes the update function using  concatenation-based MLP, which is injective with respect to its first argument. Hence, \(\mathbf{h}_a^{(1)}(G) \neq \mathbf{h}_a^{(1)}(H)\).

% \item  Case 1 (Structural Difference): $(a,b) \in E(G)$ but $(a,b) \notin E(H)$. In $G$, $b \in N_G(a)$. Thus, $\mathbf{h}_{\{a,b\}}^{(0)}(G)$ exists and is included in the multiset input to ${AGGR}_t$ for target node $a$:
% ${M}_a^{(G)} = \left\{\!\!\left\{ \mathbf{h}_{\{a,w\}}^{(0)} \mid w \in N_G(a) \right\}\!\!\right\} \ni \mathbf{h}_{\{a,b\}}^{(0)}(G)$.
% Since ${AGGR}_t$ is injective (e.g., sum aggregation over distinct multisets), ${AGGR}_t\left({M}_a^{(G)}\right) \neq \text{AGGR}_t\left({M}_a^{(H)}\right)$.This propagates to $\mathbf{h}_a^{(1)}$:
% $\mathbf{h}_a^{(1)}(G) = \phi\left( \text{AGGR}_t\left({M}_a^{(G)}\right), \cdot \right) \neq \phi\left( \text{AGGR}_t\left({M}_a^{(H)}\right), \cdot \right) = \mathbf{h}_a^{(1)}(H)$.
% \end{itemize}

\emph{\textbf{Scenario 2}} represents the feature difference, such as when $(a,b) \in E(G) \cap E(H)$, but $\mathbf{x}_a{(G)} \neq \mathbf{x}_a{(H)}$ or $\mathbf{x}_b{(G)} \neq \mathbf{x}_b{(H)}$. In this scenario, the initial node-pair representations are constructed via the injective function $\mathrm{COMB}_{\text{pair}}$ (cf. Eq.~\ref{eq4}) ensuring that
\begin{equation}
\small
\mathrm{COMB}_{\text{pair}}(\mathbf{x}_a(G), \mathbf{x}_b(G)) \neq \mathrm{COMB}_{\text{pair}}(\mathbf{x}_a(H), \mathbf{x}_b(H)).
\end{equation}
i.e., $\mathbf{h}_{\{a,b\}}^{(0)}(G) \neq \mathbf{h}_{\{a,b\}}^{(0)}(H)$. This difference  is captured by $\mathrm{AGGR}_t$ for target nodes $a$ or $b$, leading to $\mathbf{h}_a^{(1)}(G) \neq \mathbf{h}_a^{(1)}(H)$ or $\mathbf{h}_b^{(1)}(G) \neq \mathbf{h}_b^{(1)}(H)$. Thus the base case holds.

% = \mathbf{h}{{a,b}}^{(0)}(H)

Now suppose the inductive hypothesis holds for layer $l - 1$ and consider the update at layer $l$. Suppose \(c_{\{a,b\}}^{(l)}(G) \neq c_{\{a,b\}}^{(l)}(H)\),
by the 2-WL update rule (cf.~Eq.~\ref{eq2}), this difference must arise from a mismatch in the below multisets
\[
\{\!\!\{ c_{\{w,b\}}^{(l-1)} \mid w \in V \}\!\!\} \quad \text{or} \quad \{\!\!\{ c_{\{a,w\}}^{(l-1)} \mid w \in V \}\!\!\}.
\]
Without loss of generality, assume the former differs between \(G\) and \(H\). The difference in the multiset $\{\!\!\{ c_{\{w,b\}}^{(l-1)} \mid w \in V \}\!\!\}$ implies the existence of at least one node $w^*$ such that $c_{\{w^*,b\}}^{(l-1)}(G) \neq c_{\{w^*,b\}}^{(l-1)}(H)$.
By the inductive hypothesis, this implies \(\mathbf{h}_s^{(l)}(G) \neq \mathbf{h}_s^{(l)}(H)\) for some \(s \in \{w^*, b\}\).
By injectivity of \(\mathrm{COMB}_{\text{pair}}\) and \(\mathrm{AGGR}_t\), this difference further propagates to a target node \(t\) where $t \in \mathcal{N}(s)$,  leading to \(\mathbf{h}_t^{(l+1)}(G) \neq \mathbf{h}_t^{(l+1)}(H)\). Thus the induction holds.

Having established the general result, we now use an example to show that the LGAN can distinguish graphs that the set-based 2-WL cannot. For graphs $G$ and $H$ in Figure~\ref{fig1} with identical initial node labels, any connected node pair (e.g., \(\{1,3\}\)) has the same local statistics under the set-based 2-WL, that is, node 1 connects to node 3 and one other node, and is unconnected to the remaining nodes, while node 3 is symmetric. The multisets \(\{\!\{\{w,3\}\}\!\}\) and \(\{\!\{\{1,w\}\}\!\}\) therefore contain the same number of connected and unconnected pairs in both graphs. As a result, the identical node labels and matching multiset statistics lead to the same node-pair representations. The same process applies to unconnected node pairs. Thus, the set-based 2-WL cannot distinguish $G$ and $H$, but the LGAN can (Figure~\ref{fig1}), proving more expressive.
\end{proof}

\subsection{The Theoretical Connection to $2$-FWL}
The standard $k$-FWL algorithm colors ordered
$k$-tuples and replaces one coordinate of the tuple by all
possible nodes (cf.\ Eq.~\eqref{eq3}), which incurs prohibitive computational costs. To mitigate this, we follow the common practice of converting ordered tuples into unordered sets, obtaining the set-based $k$-FWL update as
\begin{subequations}\label{eq:setkfwl}
\begin{align}
\label{eq:setkfwl-neigh}
&c_{\mathbf v,w}^{(l)}
  \;\leftarrow\;
     \{\!\{\,c_{\mathbf v(1\!\to\!w)}^{(l-1)},\dots,
                         c_{\mathbf v(k\!\to\!w)}^{(l-1)}\}\!\} \\[4pt]
\label{eq:setkfwl-hash}
&c_{\mathbf v}^{(l)}
  \;=\;
    \text{HASH}\Bigl(
      c_{\mathbf v}^{(l-1)},\;
      \bigl\{\!\bigl\{\,c_{\mathbf v,w}^{(l)} \;\big|\; w\in V\bigr\}\!\bigr\}
    \Bigr)
\end{align}
\end{subequations}

Fix the $k=2$ case and choose an unordered node pair
$\mathbf v=\{p,q\}$.
In Eq.~\eqref{eq:setkfwl-neigh}, the 2-FWL traverses every $w\in V$ to form the
multiset
\(
  c^{(l)}_{\mathbf v,w}
  =
  \{\!\{c^{(l-1)}_{\{w,q\}},\,c^{(l-1)}_{\{p,w\}}\}\!\}.
\)
Observe that

\begin{itemize}
\item The collection
      $\{\!\{c^{(l-1)}_{\{w,q\}},\,c^{(l-1)}_{\{p,w\}}\}\!\}$
      exactly mirrors the set of node pairs incident to
      $w$ and $\{p,q\}$ in the LGAN.
\item
      $\displaystyle
      \text{AGGR}_n\bigl(\{\!\{\mathbf h^{(l-1)}_{\{p,q\}}\}\!\}\bigr)$
      plays the same role as the
      $c^{(l-1)}_{\{p,q\}}$
      inside the 2-FWL hash.
\item
      When we regard the auxiliary node $w$ as the target $t$ in the LGAN,
      the multiset
      $\{\!\{\mathbf h^{(l-1)}_{\{w,q\}},\mathbf h^{(l-1)}_{\{p,w\}}\}\!\}$
      becomes precisely the input to $\text{AGGR}_{t}$.
\end{itemize}

Hence, each set-based 2-FWL update on the pair $\{p,q\}$ can be locally simulated by a single LGAN update on the 1-hop induced subgraph around the target node $t=w$. Intuitively, the LGAN can be viewed as a sparse and local variant of the set-based 2-FWL, replacing global tuple updates with line graph aggregation over ego networks.

\subsection{The Time Complexity Analysis}
The LGAN attains the expressive benefits of \(2\text{-}\mathrm{FWL}\) \textbf{without its cubic cost}.
Let \(d_v\) be the degree of node \(v\). A single LGAN layer aggregates \(d_v\) target-neighbor edges and at most \(\binom{d_v}{2} = O(d_v^2)\) neighbor-neighbor edges, giving a per-node cost of  ~\(O(d_v + d_v^2)\).
Summing over all nodes yields a total complexity of
\begin{equation}
O\left( \sum_{v} d_v + \sum_{v} d_v^2 \right) = O\left( m + \sum_{v} d_v^2 \right),
\end{equation}
ranging from \(O(n)\) on  sparse graphs to \(O(n^3)\) in the fully connected worst case-yet with a much smaller constant factor than the global \(2\text{-}\mathrm{FWL}\) due to the locality of the LGAN.

In contrast, higher-order GNNs such as the $k$-GNNs~\cite{morris2019goneural} and tensor-based models~\cite{maron2019universality,maron2019provably} incur $O(n^k)$ complexity.
In practice, the LGAN achieves linear runtime on sparse graphs, compared to the $O(n^3)$ cost of the 2-WL with lower expressivity.

\begin{table*}[ht]
\begin{center}
\begin{tabular}{llcccccc}
\toprule
\multirow{4}{*}{{Datasets}} &  & MUTAG & PTC(MR) & PROTEINS & IMDB-B & IMDB-M & COLLAB \\
 & \# graphs & 188 &344 &1113 &1000 &1500 &5000 \\
 & \# classes & 2 & 2 &2 &2 &3 &3  \\
 & Avg. \# nodes & 18 &26 &39 &20 &13 &74  \\
\midrule
\multirow{4}{*}{{Kernels}}
 & RW & 79.2$\pm$2.1 & 55.9$\pm$0.3 & 59.6$\pm$0.1 & \text{N/A} & \text{N/A} & \text{N/A}  \\
 & SP & 81.7$\pm$ & 58.9$\pm$ & 76.4$\pm$ & 59.2$\pm$ & 40.5$\pm$ & \text{N/A} \\
 & PK & 76.0$\pm$2.7 & 59.5$\pm$2.4 & 73.7$\pm$0.7 & \text{N/A} & \text{N/A} & \text{N/A}  \\
 & 2-WL & 77.0$\pm$ & 61.9$\pm$ & 75.2$\pm$ & 72.6$\pm$ & 50.6$\pm$ & \text{N/A} \\
\midrule
\multirow{4}{*}{{GNNs}}
 & DCNN & \text{N/A} & \text{N/A} & 61.3$\pm$1.6 & 49.1$\pm$1.4 & 33.5$\pm$1.4 & 52.1$\pm$0.7  \\
 & \textsc{Patchy}\textsc{San} & \underline{92.6$\pm$4.2} & 60.0$\pm$4.8 & 75.9$\pm$2.8 & 71.0$\pm$2.2 & 45.2$\pm$2.8 & 72.6$\pm$2.2  \\
 & DGCNN & 85.8$\pm$1.7 & 58.6$\pm$2.5 & 75.5$\pm$0.9 & 70.0$\pm$0.9 & 47.8$\pm$0.9 & 73.8$\pm$0.5 \\
 & GIN & 89.4$\pm$5.6 & 64.6$\pm$7.0 & 76.2$\pm$2.8 & 75.1$\pm$5.1 & 52.3$\pm$2.8 & 80.2$\pm$1.9  \\
\midrule
\multirow{4}{*}{{$k$-WL-based}}
 & 1-2-3 GNN & 86.1$\pm$ & 60.9$\pm$ & 75.5$\pm$ & 74.2$\pm$ & 49.5$\pm$ & \text{N/A} \\
 & PPGN & 90.6$\pm$8.7 & 66.2$\pm$6.5 & \underline{77.2$\pm$4.7} & 73.0$\pm$5.8 & 50.5$\pm$3.6 & 81.4$\pm$1.4  \\
 & $\delta$-2-LWL & \text{N/A} & \text{N/A} & 75.1$\pm$0.3 & 73.3$\pm$0.5 & 50.2$\pm$0.6 & \text{N/A}  \\
 & CW Networks & \textbf{92.7$\pm$6.1} & \textbf{68.2$\pm$5.6} & 77.0$\pm$4.3 & 75.6$\pm$3.7 & 52.7$\pm$3.1 & \text{N/A}  \\
\midrule
\multirow{2}{*}{{Ours}}
 & LGAN &  92.5$\pm$6.3 &  \underline{67.4$\pm$6.2} &  \textbf{77.3$\pm$3.7} &  \textbf{76.7$\pm$4.0} &  \underline{53.3$\pm$3.2} &  \textbf{82.8$\pm$1.5} \\
 & LGAN-\emph{res} &  92.0$\pm$5.8  &  66.6$\pm$7.0 &  76.8$\pm$3.6 &  \underline{76.6$\pm$3.1}  &  \textbf{53.5$\pm$2.7} & \underline{82.7$\pm$1.7} \\
\bottomrule
\end{tabular}
\caption{Test set classification accuracies (\%). The mean accuracy and standard deviation are reported. Best performances are highlighted in bold, and second-best performances are underlined. \text{N/A} means Not Available.}
\label{table1}
\end{center}
\end{table*}

\section{Experiments}
In this section, we evaluate the LGAN on six graph classification benchmarks from social networks, bioinformatics, and chemistry~\cite{yanardag2015deep}.  We also demonstrate its interpretability via edge attribution using Integrated Gradients~\cite{IG}.
\subsection{Experimental Setups}

\noindent\textbf{Setup.} We adopt a 10-fold cross-validation strategy with stratified splits that preserve the label distribution across folds. We report our results following the evaluation protocol described in~\cite{xu2018powerful}. Node degree one-hot encodings are used for social networks, and provided node labels or attributes for bioinformatics and chemical datasets.

\noindent\textbf{Hyper-parameters.} We tune the number of layers, hidden dimension, dropout, and learning rate for each dataset.

\noindent\textbf{Baselines.}  We compare the  proposed LGAN against three representative  categories of models: graph kernels, standard GNNs, and $k$-WL-based methods. Specifically, the graph kernel baselines include the random walk kernel (RW)~\cite{vishwanathan2010graph}, shortest path kernel (SP)~\cite{borgwardt2005shortest}, propagation kernel (PK)~\cite{neumann2016propagation},  and 2-WL kernel~\cite{morris2019goneural}. The GNN baselines include the DCNN~\cite{atwood2016diffusion}, PATCHY-SAN~\cite{niepert2016learningconvolutionalneuralnetworks}, DGCNN~\cite{zhang2018end} and GIN~\cite{xu2018powerful}. The $k$-WL-based methods include the 1-2-3 GNN~\cite{morris2019goneural}, PPGN~\cite{maron2019provably}, $\delta$-$k$-LWL~\cite{morris2020weisfeiler}, and CW Networks~\cite{bodnar2021weisfeiler}. 
\subsection{Experimental Results and Analysis}
\textbf{Table~\ref{table1}} presents classification accuracies on six benchmark datasets. The proposed LGAN and LGAN-\emph{res} achieve superior or comparable performance across most datasets. Compared to kernel-based methods and standard GNNs, our models offer consistently better accuracy. While several $k$-WL-based models attain strong results, they  face scalability issues on large graphs (e.g., COLLAB). Overall, the LGAN strikes a favorable balance between accuracy and efficiency.
\subsection{Interpretability via Edge Attribution}
We assess the interpretability of the LGAN using edge attribution (Integrated Gradients) to determine whether it can identify critical substructures contributing to classification.

Figures~\ref{fig3}(a) and (b) present two synthetic graphs $G^{\prime}$ and $H^{\prime}$, both of which are indistinguishable by the 2-WL due to having identical initial node features and neighborhood distributions. Similar to the pair in Figure~\ref{fig1}, they differ only in the presence or absence of a triangle, which defines their class label. The LGAN successfully assigns high importance to the triangle edges, highlighting its ability to capture essential neighborhood interactions.

Figures~\ref{fig3}(c) and (d) show two structurally similar molecules from the Mutagenicity dataset: 2-nitroanisole (mutagen) and 2-nitrobenzyl alcohol (non-mutagen). Both contain a nitro group (NO$_2$), which is often associated with mutagenic activity. However, substituent difference (\ce{-OCH3} vs. \ce{-CH2OH}) determines mutagenicity. The LGAN correctly emphasizes edges near critical functional groups, revealing chemically meaningful patterns. 
In summary, the LGAN offers fine-grained interpretability beyond the $k$-WL-based models that lack localized attribution.
% Figures~\ref{fig3}(c) and (d) show two structurally similar molecules from the Mutagenicity dataset, 2-nitroanisole (mutagenic) and 2-nitrobenzyl alcohol (non-mutagenic). Both contain a nitro group (NO$_2$), commonly associated with mutagenic activity. 
% Their key difference lies in the substituent: methoxy (\ce{-OCH3}) vs. hydroxymethyl (\ce{-CH2OH}), which affects mutagenicity. Through its target-aware aggregation, the LGAN highlights edges near the critical functional groups and reveals chemically meaningful patterns, offering more fine-grained interpretability than the $k$-WL-based models,  which lack localized attribution mechanisms.
% However, their key difference lies in the substituent: methoxy (\ce{-OCH3}) vs. hydroxymethyl (\ce{-CH2OH}), which affects mutagenicity. This substitution reduces mutagenicity, and the LGAN is able to highlight edges near the critical  functional groups,  capturing chemically meaningful distinctions.

% Overall, the LGAN offers fine-grained interpretability through its target-aware aggregation, making it more suitable for interpretability than the $k$-WL-based models, which lack localized attribution mechanisms.

\begin{figure}[h]
  \centering
  \begin{subfigure}{0.39\columnwidth}
    \includegraphics[width=\linewidth]{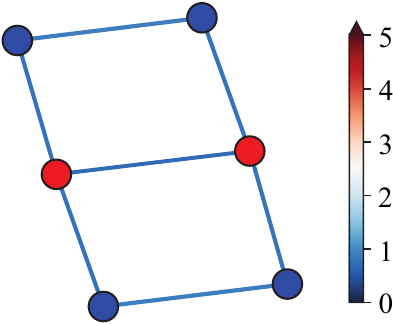}
    \caption{G': No triangle}
    \label{fig3_a}
  \end{subfigure} \hfill
\begin{subfigure}{0.39\columnwidth}
    \includegraphics[width=\linewidth]{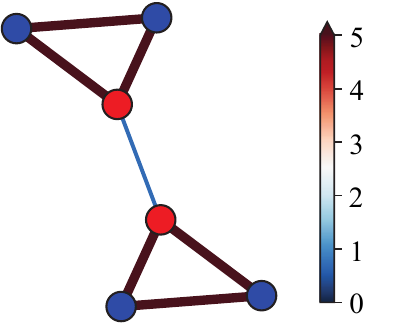}
    \caption{H': With triangle}
    \label{fig3_b}
  \end{subfigure}

\begin{subfigure}{0.42\columnwidth}
    \includegraphics[width=\linewidth]{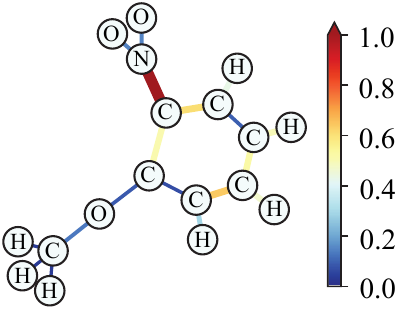}
    \caption{Mutagen molecule}
    \label{fig3_c}
  \end{subfigure} \hfill
\begin{subfigure}{0.42\columnwidth}
    \includegraphics[width=\linewidth]{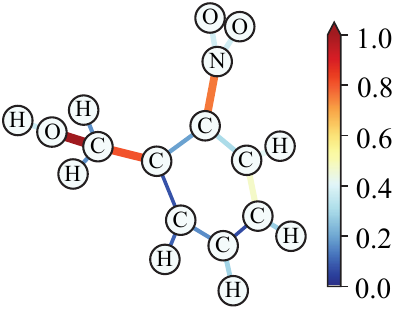}
    \caption{Non-mutagen molecule}
    \label{fig3_d}
  \end{subfigure}
\caption{Edge visualizations with Integrated Gradients (redder and thicker edges represent higher importance).}
\label{fig3}
\end{figure}
\section{Conclusion}
In this paper, we introduced the LGAN, a novel GNN that surpasses 2-WL expressivity through localized line graph aggregation. The LGAN is both efficient and interpretable, and performs competitively on graph classification tasks against state-of-the-art methods, including the $k$-WL-based models. Moreover, the LGAN can be naturally applied to node- and edge-level tasks, as its target-aware aggregation mechanism is general. In future works, we plan to extend the LGAN to more diverse applications, such as complex structural alignment and relational modeling~\cite{bai2025aegk}.

\section*{Acknowledgments}
This work is supported by the National Natural Science Foundation of China (No. 62576371, T2122020, 62576198, and 62471288).
% AAAI is especially grateful to Peter Patel Schneider for his work in implementing the original aaai.sty file, liberally using the ideas of other style hackers, including Barbara Beeton. We also acknowledge with thanks the work of George Ferguson for his guide to using the style and BibTeX files --- which has been incorporated into this document --- and Hans Guesgen, who provided several timely modifications, as well as the many others who have, from time to time, sent in suggestions on improvements to the AAAI style. We are especially grateful to Francisco Cruz, Marc Pujol-Gonzalez, and Mico Loretan for the improvements to the Bib\TeX{} and \LaTeX{} files made in 2020.

% The preparation of the \LaTeX{} and Bib\TeX{} files that implement these instructions was supported by Schlumberger Palo Alto Research, AT\&T Bell Laboratories, Morgan Kaufmann Publishers, The Live Oak Press, LLC, and AAAI Press. Bibliography style changes were added by Sunil Issar. \verb+\+pubnote was added by J. Scott Penberthy. George Ferguson added support for printing the AAAI copyright slug. Additional changes to aaai2026.sty and aaai2026.bst have been made by Francisco Cruz, Marc Pujol-Gonzalez, and Mico Loretan.

% \bigskip
% \noindent Thank you for reading these instructions carefully. We look forward to receiving your electronic files!
\bibliography{aaai2026}

@misc{xu2018powerful,
      title={How Powerful Are Graph Neural Networks?}, 
      author={Keyulu Xu and Weihua Hu and Jure Leskovec and Stefanie Jegelka},
      year={2019},
      eprint={1810.00826},
      archivePrefix={arXiv},
      primaryClass={cs.LG},
      url={https://arxiv.org/abs/1810.00826}, 
}

@misc{morris2019goneural,
      title={Weisfeiler and Leman Go Neural: Higher-Order Graph Neural Networks}, 
      author={Christopher Morris and Martin Ritzert and Matthias Fey and William L. Hamilton and Jan Eric Lenssen and Gaurav Rattan and Martin Grohe},
      year={2021},
      eprint={1810.02244},
      archivePrefix={arXiv},
      primaryClass={cs.LG},
      url={https://arxiv.org/abs/1810.02244}, 
}

@article{shervashidze2011weisfeiler,
author = {Shervashidze, Nino and Schweitzer, Pascal and van Leeuwen, Erik Jan and Mehlhorn, Kurt and Borgwardt, Karsten M.},
title = {Weisfeiler-Lehman Graph Kernels},
year = {2011},
publisher = {JMLR.org},
volume = {12},
number = {9},
issn = {1532-4435},
journal = {Journal of Machine Learning Research},
pages = {2539–2561},
}

@misc{maron2019provably,
      title={Provably Powerful Graph Networks}, 
      author={Haggai Maron and Heli Ben-Hamu and Hadar Serviansky and Yaron Lipman},
      year={2020},
      eprint={1905.11136},
      archivePrefix={arXiv},
      primaryClass={cs.LG},
      url={https://arxiv.org/abs/1905.11136}, 
}

@misc{sato2020survey,
      title={A Survey on the Expressive Power of Graph Neural Networks}, 
      author={Ryoma Sato},
      year={2020},
      eprint={2003.04078},
      archivePrefix={arXiv},
      primaryClass={cs.LG},
      url={https://arxiv.org/abs/2003.04078}, 
}

@inproceedings{huang2021short,
      author={Huang, Ningyuan Teresa and Villar, Soledad},
      booktitle={Proceedings of the IEEE International Conference on Acoustics, Speech and Signal Processing}, 
      title={A Short Tutorial on the Weisfeiler-Lehman Test and Its Variants}, 
      year={2021},
      pages={8533-8537},
}

@article{whitney1992congruent,
  title={Congruent Graphs and the Connectivity of Graphs},
  author={Whitney, Hassler},
  journal={Hassler Whitney Collected Papers},
  pages={61--79},
  year={1992},
  publisher={Springer}
}

@article{roussopoulos1973max,
  title={A Max $\{$m, n$\}$ Algorithm for Determining the Graph H from Its Line Graph G},
  author={Roussopoulos, Nicholas D},
  journal={Information Processing Letters},
  volume={2},
  number={4},
  pages={108--112},
  year={1973},
  publisher={Elsevier}
}

@article{lehot1974optimaldetectlinegraph,
  title={An Optimal Algorithm to Detect a Line Graph and Output Its Root Graph},
  author={Lehot, Philippe GH},
  journal={Journal of the Association for Computing Machinery},
  volume={21},
  number={4},
  pages={569--575},
  year={1974},
  publisher={ACM New York, NY, USA}
}

@article{cai2021line,
  title={Line Graph Neural Networks for Link Prediction},
  author={Cai, Lei and Li, Jundong and Wang, Jie and Ji, Shuiwang},
  journal={IEEE Transactions on Pattern Analysis and Machine Intelligence},
  volume={44},
  number={9},
  pages={5103--5113},
  year={2021},
  publisher={IEEE}
}

@article{zhang2023line,
  title={Line Graph Contrastive Learning for Link Prediction},
  author={Zhang, Zehua and Sun, Shilin and Ma, Guixiang and Zhong, Caiming},
  journal={Pattern Recognition},
  volume={140},
  pages={109537},
  year={2023},
  publisher={Elsevier}
}

@article{liang2025line,
  title={Line Graph Neural Networks for Link Weight Prediction},
  author={Liang, Jinbi and Pu, Cunlai and Shu, Xiangbo and Xia, Yongxiang and Xia, Chengyi},
  journal={Physica A: Statistical Mechanics and its Applications},
  volume={661},
  pages={130406},
  year={2025},
  publisher={Elsevier}
}

@misc{bodnar2021weisfeiler,
      title={Weisfeiler and Lehman Go Cellular: CW Networks}, 
      author={Cristian Bodnar and Fabrizio Frasca and Nina Otter and Yu Guang Wang and Pietro Liò and Guido Montúfar and Michael Bronstein},
      year={2022},
      eprint={2106.12575},
      archivePrefix={arXiv},
      primaryClass={cs.LG},
      url={https://arxiv.org/abs/2106.12575}, 
}

@misc{maron2019universality,
      title={On the Universality of Invariant Networks}, 
      author={Haggai Maron and Ethan Fetaya and Nimrod Segol and Yaron Lipman},
      year={2019},
      eprint={1901.09342},
      archivePrefix={arXiv},
      primaryClass={cs.LG},
      url={https://arxiv.org/abs/1901.09342}, 
}

@misc{morris2020weisfeiler,
      title={Weisfeiler and Leman Go Sparse: Towards Scalable Higher-Order Graph Embeddings}, 
      author={Christopher Morris and Gaurav Rattan and Petra Mutzel},
      year={2020},
      eprint={1904.01543},
      archivePrefix={arXiv},
      primaryClass={cs.DS},
      url={https://arxiv.org/abs/1904.01543}, 
}

@misc{chen2017supervised,
      title={Supervised Community Detection with Line Graph Neural Networks}, 
      author={Zhengdao Chen and Xiang Li and Joan Bruna},
      year={2020},
      eprint={1705.08415},
      archivePrefix={arXiv},
      primaryClass={stat.ML},
      url={https://arxiv.org/abs/1705.08415}, 
}

@article{weisfeiler1968reduction,
  title={The Reduction of a Graph to Canonical Form and the Algebra Which Appears Therein},
  author={Weisfeiler, Boris and Leman, Andrei},
  journal={Nauchno‑Technicheskaya Informatsia},
  volume={2},
  number={9},
  pages={12--16},
  year={1968}
}

@misc{grohe2015pebble,
      title={Pebble Games and Linear Equations}, 
      author={Martin Grohe and Martin Otto},
      year={2015},
      eprint={1204.1990},
      archivePrefix={arXiv},
      primaryClass={cs.LO},
      url={https://arxiv.org/abs/1204.1990}, 
}

@book{grohe2017descriptive,
  title={Descriptive Complexity, Canonisation, and Definable Graph Structure Theory},
  author={Grohe, Martin},
  year={2017},
  publisher={Cambridge University Press}
}

@misc{grohe2021logic,
      title={The Logic of Graph Neural Networks}, 
      author={Martin Grohe},
      year={2022},
      eprint={2104.14624},
      archivePrefix={arXiv},
      primaryClass={cs.LG},
      url={https://arxiv.org/abs/2104.14624}, 
}

@misc{xu2018skipcat,
      title={Representation Learning on Graphs with Jumping Knowledge Networks}, 
      author={Keyulu Xu and Chengtao Li and Yonglong Tian and Tomohiro Sonobe and Ken-ichi Kawarabayashi and Stefanie Jegelka},
      year={2018},
      eprint={1806.03536},
      archivePrefix={arXiv},
      primaryClass={cs.LG},
      url={https://arxiv.org/abs/1806.03536}, 
}

@misc{zaheer2017deep,
      title={Deep Sets}, 
      author={Manzil Zaheer and Satwik Kottur and Siamak Ravanbakhsh and Barnabas Poczos and Ruslan Salakhutdinov and Alexander Smola},
      year={2018},
      eprint={1703.06114},
      archivePrefix={arXiv},
      primaryClass={cs.LG},
      url={https://arxiv.org/abs/1703.06114}, 
}

@article{cai1992optimal,
  title={An Optimal Lower Bound on the Number of Variables for Graph Identification},
  author={Cai, Jin-Yi and F{\"u}rer, Martin and Immerman, Neil},
  journal={Combinatorica},
  volume={12},
  number={4},
  pages={389--410},
  year={1992},
  publisher={Springer}
}

@misc{he2016deep,
      title={Deep Residual Learning for Image Recognition}, 
      author={Kaiming He and Xiangyu Zhang and Shaoqing Ren and Jian Sun},
      year={2015},
      eprint={1512.03385},
      archivePrefix={arXiv},
      primaryClass={cs.CV},
      url={https://arxiv.org/abs/1512.03385}, 
}

@inproceedings{yanardag2015deep,
      author = {Yanardag, Pinar and Vishwanathan, S.V.N.},
      title = {Deep Graph Kernels},
      year = {2015},
      booktitle = {Proceedings of the ACM SIGKDD Conference on Knowledge Discovery and Data Mining},
      pages = {1365–1374},
}

@article{neumann2016propagation,
  title={Propagation Kernels: Efficient Graph Kernels from Propagated Information},
  author={Neumann, Marion and Garnett, Roman and Bauckhage, Christian and Kersting, Kristian},
  journal={Machine Learning},
  volume={102},
  number = {2},
  pages={209--245},
  year={2016},
  publisher={Springer}
}

@inproceedings{zhang2018end,
  title={An End-to-End Deep Learning Architecture for Graph Classification},
  author={Zhang, Muhan and Cui, Zhicheng and Neumann, Marion and Chen, Yixin},
  booktitle={Proceedings of the AAAI Conference on Artificial Intelligence},
  pages={4438--4445},
  year={2018}
}

@inproceedings{borgwardt2005shortest,
  title = {Shortest-Path Kernels on Graphs},
  author={Borgwardt, Karsten M and Kriegel, Hans-Peter},
  booktitle={Proceedings of the IEEE International Conference on Data Mining},
  pages = {74–81},
  year={2005},
}

@misc{atwood2016diffusion,
      title={Diffusion-Convolutional Neural Networks}, 
      author={James Atwood and Don Towsley},
      year={2016},
      eprint={1511.02136},
      archivePrefix={arXiv},
      primaryClass={cs.LG},
      url={https://arxiv.org/abs/1511.02136}, 
}

@misc{vishwanathan2010graph,
      title={Graph Kernels}, 
      author={S. V. N. Vishwanathan and Karsten M. Borgwardt and Imre Risi Kondor and Nicol N. Schraudolph},
      year={2008},
      eprint={0807.0093},
      archivePrefix={arXiv},
      primaryClass={cs.LG},
      url={https://arxiv.org/abs/0807.0093}, 
}

@misc{niepert2016learningconvolutionalneuralnetworks,
      title={Learning Convolutional Neural Networks for Graphs}, 
      author={Mathias Niepert and Mohamed Ahmed and Konstantin Kutzkov},
      year={2016},
      eprint={1605.05273},
      archivePrefix={arXiv},
      primaryClass={cs.LG},
      url={https://arxiv.org/abs/1605.05273}, 
}

@article{bai2025aegk,
  title={AEGK: Aligned Entropic Graph Kernels Through Continuous-Time Quantum Walks},
  author={Bai, Lu and Cui, Lixin and Li, Ming and Ren, Peng and Wang, Yue and Zhang, Lichi and Yu, Philip S. and Hancock, Edwin R.},
  journal={IEEE Transactions on Knowledge and Data Engineering},
  year={2025},
  volume={37},
  number={3},
  pages={1064-1078},
  publisher={IEEE}
}

@inproceedings{qin2025,
  title     = {HA-SCN: Learning Hierarchical Aligned Subtree Convolutional Networks for Graph Classification},
  author    = {Qin, Xinya and Bai, Lu and Cui, Lixin and Li, Ming and Du, Hangyuan and Wang, Yue and Hancock, Edwin R.},
  booktitle = {Proceedings of the International Joint Conference on
               Artificial Intelligence},
  pages     = {3245--3253},
  year      = {2025},
}

@article{cui2024,
  title={Learning Aligned Vertex Convolutional Networks for Graph Classification},
  author={Cui, Lixin and Bai, Lu and Bai, Xiao and Wang, Yue and Hancock, Edwin R.},
  journal={IEEE Transactions on Neural Networks and Learning Systems},
  volume={35},
  number={4},
  pages={4423--4437},
  year={2024},
  publisher={IEEE}
}

@article{bai2023,
  title={Learning Graph Convolutional Networks Based on Quantum Vertex Information Propagation},
  author={Bai, Lu and Jiao, Yuhang and Cui, Lixin and Rossi, Luca and Wang, Yue and Yu, Philip S. and Hancock, Edwin R.},
  journal={IEEE Transactions on Knowledge and Data Engineering},
  volume={35},
  number={2},
  pages={1747--1760},
  year={2023},
  publisher={IEEE}
}

@inproceedings{ICPR2014,
  author={Bai, Lu and Ren, Peng and Hancock, Edwin R.},
  booktitle={Proceedings of International Conference on Pattern Recognition}, 
  title={A Hypergraph Kernel from Isomorphism Tests}, 
  year={2014},
  pages={3880-3885}
}

@article{2016pr,
  title={Depth-Based Hypergraph Complexity Traces from Directed Line Graphs},
  author={Bai, Lu and Escolano, Francisco and Hancock, Edwin R.},
  journal={Pattern Recognition},
  volume={54},
  pages={229--240},
  year={2016},
  publisher={Elsevier}
}

@inproceedings{IG,
  author = {Sundararajan, Mukund and Taly, Ankur and Yan, Qiqi},
  title = {Axiomatic Attribution for Deep Networks},
  year = {2017},
  booktitle = {Proceedings of the International Conference on Machine Learning},
  pages = {3319–3328}
}
\end{document}